%% file: main_arxiv.tex
\documentclass{article}
\usepackage[utf8]{inputenc}

\input{packages}

\input{symbols}
\input{commands}

\title{Online Model Selection: a Rested Bandit Formulation}

\author[1]{Leonardo Cella\thanks{correspondence to: leonardocella@gmail.com}}
\author[2]{Claudio Gentile}
\author[1,3]{Massimiliano Pontil}

\affil[1]{Italian Institute of Technology, Genoa, Italy}
\affil[2]{Google Research, New York, USA}
\affil[3]{University College London, United Kingdom}

\makeindex
\begin{document}
\maketitle
\input{abstract_arxiv}

\input{introduction_arxiv}
\input{related_arxiv}
\input{learningSetting_arxiv}
\input{lowerBound_arxiv}

\input{bernsteinEstimators_arxiv}

\input{etc_arxiv}
\input{cl_arxiv}
\input{conclusions}


\input{main_arxiv.bbl}
\onecolumn

\appendix
\input{appendix_arxiv}

\end{document}

%% file: packages.tex
\usepackage{algorithm}
\usepackage{algorithmic}
\usepackage{amsfonts} 
\usepackage{amsmath} 
\usepackage{amsthm} 
\usepackage{multicol}
\usepackage{mathtools} 
\usepackage{tikz}
\usepackage{url}
\usepackage{booktabs}
\usepackage{authblk}
\usepackage{amsfonts,amssymb}
\usepackage{nicefrac}
\usepackage{microtype}
\usepackage{algorithm}
\usepackage{algorithmic}
\usepackage{subcaption}
\usepackage[square,numbers]{natbib}

%% file: symbols.tex
\newcommand{\X}{\mathcal{X}}
\newcommand{\Xh}{\widehat{X}}
\newcommand{\Xt}{\widetilde{X}}

\newcommand{\Var}{\mathbb{V}}
\newcommand{\tsub}{\tau_{sub}}
\newcommand{\tout}{\tau_{\text{out}}}
\newcommand{\tb}{\overline{t}}

\newcommand{\sbar}{\overline{s}}

\newcommand{\re}{\mathbb{R}}

\newcommand{\reg}[1]{R_T^{#1}}

\newcommand{\RS}{REST-SURE }

\newcommand{\Pro}{\mathbb{P}}
\newcommand{\olg}{\overline{\log}}
\newcommand{\nb}{\overline{n}}

\newcommand{\muu}{\underline{\mu}}
\newcommand{\mut}{\widetilde{\mu}}
\newcommand{\muh}{\widehat{\mu}}

\newcommand{\Km}{\mathcal{K}^{-}}

\newcommand{\K}{\mathcal{K}}

\newcommand{\iT}{i_T^*}

\newcommand{\iN}{i_{\nb}^*}
\newcommand{\iout}{i_{\text{out}}}
\newcommand{\ib}{\overline{i}}
\newcommand{\I}{\mathbb{I}}

\newcommand{\F}{\mathcal{F}}
\newcommand{\E}{\mathbb{E}}
\newcommand{\Dmu}{\Delta\muh}
\newcommand{\DX}{\Delta\Xh}
\newcommand{\DeltaTilde}{\widetilde{\Delta}}

\newcommand{\cb}{\overline{c}}
\newcommand{\bh}{\widehat{\beta}}

\newcommand{\ah}{\widehat{\alpha}}
\newcommand{\ab}{\alpha}
\newcommand{\A}{\mathcal{A}}
\newcommand{\As}{\mathcal{A}^*}

%% file: commands.tex
\newtheorem{lemma}{Lemma}
\newtheorem{assumption}{Assumption}
\newtheorem{theorem}{Theorem}
\newtheorem{corollary}{Corollary}

\newtheorem{fact}{Fact}

\newtheorem{example}{Example}

\newcommand{\quotes}[1]{``#1''}

\DeclarePairedDelimiter\ceil{\lceil}{\rceil}
\DeclarePairedDelimiter\floor{\lfloor}{\rfloor}

\newcommand{\ignore}[1]{}

\definecolor{mycolor}{rgb}{0,0.7,0.2}

%% file: abstract_arxiv.tex
\begin{abstract}
\noindent Motivated by a natural problem in online model selection with bandit information, we introduce and analyze a best arm identification problem in the rested bandit setting, wherein arm expected losses decrease with the number of times the arm has been played. The shape of the expected loss functions is similar across arms, and is assumed to be available up to unknown parameters that have to be learned on the fly. We define a novel notion of regret for this problem, where we compare to the policy that always plays the arm having the smallest expected loss at the end of the game. We analyze an arm elimination algorithm whose regret vanishes as the time horizon 
increases. The actual rate of convergence depends in a detailed way on the postulated functional form of the expected losses. Unlike known model selection efforts in the recent bandit literature, our algorithm exploits the specific structure of the problem to learn the unknown parameters of the expected loss function so as to identify the best arm as quickly as possible. We complement our analysis with a lower bound, indicating strengths and limitations of the proposed solution.
\end{abstract}

%% file: introduction_arxiv.tex
\section{Introduction} \label{Sec:Introduction}
Multi-armed bandits are a mathematical framework of sequential decision problems that
have received in the last two decades increasing attention (e.g., \citep{RegretMAB,cella2020meta,MAB,kuzborskij2019efficient,lattimore2018bandit,RobbinsOriginal,RobbinsReviewed}), becoming a prominent area of machine learning and statistics. This framework consists of a sequence of $T$ interactions (or {\em rounds}) between a learning agent $\pi$ and an unknown environment. During each round the learner picks an action from a set of options $\K$ to pull, which are usually referred to as {\em arms}, and the environment consequently generates a feedback (e.g., in the form of a loss value) associated with the chosen action/pulled arm.
%
%
Multi-armed bandits have found applications in a wide variety of domains including clinical trials, online advertising, and product recommendation.

In the standard i.i.d. stochastic bandit setting (e.g., \citep{auer2002finite}), the feedback generated when pulling an arm is modeled as a random variable sampled from a prescribed distribution associated with the selected arm, and this distribution remains the same across rounds.
%
In contrast, in this paper we are interested in a {\em non-stationary} stochastic bandit setting called \textit{rested bandits}~\cite{allesiardo2017non,besbes2014stochastic,cella2020stochastic,kleinberg2018recharging,levine2017rotting,seznec2018rotting}. 
Here, the feedback/losses received upon pulling arms are not i.i.d. anymore.  
Instead, the distribution of losses changes as a function of the number of times each arm has been pulled so far.
%
As a relevant example, assume the expected loss of action $i \in \K$ at a given round takes the parametric form
\begin{equation}\label{e:model}
\frac{\alpha_i}{\sqrt{\tau}} + \beta_i~,
\end{equation}
where $\tau$ is the number of times arm $i$ has been pull up to that round, and $\alpha_i$ and $\beta_i$ are unknown parameters which are specific to that arm.

Considering decreasing expected losses is reasonable whenever the properties of the chosen arm improve as we allocate resources to them. For instance, this is the case in scenarios where the goal is to find the best talent in a pool of candidates, say, the most valuable worker to train in an online labor platform having limited training time. 

A striking motivation behind assumptions like (\ref{e:model}) is the study of online model selection problems with bandit feedback. Here, at each round the only observed feedback is the one associated with the selected arm, where each arm represents a learning device, so that arms could also be referred to as {\em base learners}. 
The online model selection problem when we restrict to base learners which are themselves bandit algorithms has recently received a lot of attention~(e.g., \cite{alns17,foster2019model,fgmz20,pacchiano2020model}). Yet, we would like to emphasize that, in our setting, the base learners could be any generic learning devices (like different neural network architectures) that satisfy Equation~(\ref{e:model}). 
The parameters $\alpha_i$ and $\beta_i$ in (\ref{e:model}) may therefore quantify relevant properties of such models. In a standard statistical learning setting, parameter $\alpha_i$ can quantify the complexity (which may or may not be known) of the $i$-th model class, $\beta_i$ might encode the representational power of that class in the form of the statistical risk of the best-in-class hypothesis (which is typically unknown), while the dependence on $1/\sqrt{\tau}$ is meant to suggest a plausible behavior of the generalization error of the $i$-th algorithm as a function of the training set size $\tau$.
For instance, an arm $i\in\K$ with small $\alpha_i$ and large $\beta_i$ may represent an empirical risk minimizer (ERM) operating on a simple model class where the ERM has an estimation error getting small with few samples, but which only underfits the data without effectively minimizing the approximation error. Conversely, an arm $i\in\K$ with large $\alpha_i$ and small $\beta_i$ may correspond to an ERM operating on a complex model class with large estimation error (where overfitting is likely to occur) and small approximation error.

Given a budget of $T$ training samples, our specific goal is to design a strategy for online {\em selective training}, whereby at each round we have to decide which algorithm the next training example has to be fed to. This problem is of fundamental importance since, in many practical situations, performing a batch model selection (or model training) might be too computationally demanding. Thus, the goal is to design a strategy (a learning policy) that interacts with different learning algorithms with the goal of spending the budget of $T$ samples on the algorithm/model that is likely to perform best after training.
Pulling an arm corresponds to feeding the current sample to the associated algorithm, while observing the feedback corresponds to being able to estimate in an approximate manner (e.g., on a separate test set) the generalization error of the trained algorithm for that arm, this error being a decreasing function of the number of samples the chosen algorithm has so far been trained over.\\
%
%
%
\newline
{\bf Contributions.} 
%
We first propose a novel notion of regret which is suited to the online learning problem we consider here. This regret criterion frames our problem as a best arm identification problem within a rested bandit scenario. We then characterize (at least partially) the structure of the problem by proving a non-asymptotic lower bound restricted to the 2-arm case. Finally, we describe and analyze two action elimination algorithms, and show for one of the two algorithms a regret upper bound that essentially matches the above-mentioned lower bound.\\
\newline
{\bf Notation.} 
For a positive integer $N$, we abbreviate the set $\{1,\dots,N\}$ by $[N]$. 
%
We use $\E[\cdot]$ and $\Pro[\cdot]$, to denote expected value and probability measure, respectively. Moreover, for a given $\sigma$-algebra $\F$, $\E_\F[\cdot]$ and $\Pro_\F[\cdot]$ denote their conditional counterparts: $\E_\F[\cdot] = \E[\cdot\,|\,\F]$, and  $\Pro_\F[\cdot] =  \Pro[\cdot\,|\,\F]$. 

%% file: related_arxiv.tex
\section{Related works}\label{Sec:Related}
%
The problem of online model selection in bandit settings (specifically, the case where the base learners are  contextual bandit algorithms), has been investigated in a number of papers in recent years, e.g., \cite{alns17,foster2019model,fgmz20,pacchiano2020model}. 
In particular, in \cite{alns17} the authors consider a very general  class of base learners which have to satisfy reasonable stability assumptions. Additionally, to deal with the bandit information, importance weighted feedback is given to the bandit learners. 
In \cite{foster2019model,fgmz20} the emphasis is specifically on linear bandit model selection problems, where model selection operates on the input dimension \cite{foster2019model} or the amount of misspecification \cite{fgmz20}.
Similar to \cite{alns17}, 
in \cite{pacchiano2020model} the authors investigate the problem of algorithm selection in contextual bandits where contexts are stochastic. In order to bypass the stability assumption in \cite{alns17}, an additional smoothed transformation is introduced. The positive side effect induced by this additional step is the ability to feed the base learners with the original feedback with no re-weighting.

%

Unlike all the above works, we assume the expected loss of the considered base learners depends in specific ways on the number of times each base learner is selected, this dependence being known up unknown parameters that have to be estimated. More importantly, we investigate a performance metric that is different from the standard cumulative regret incurred with respect to the best allocation policy, as studied in \cite{alns17,fgmz20,foster2019model,pacchiano2020model}.

Another stream of literature which is loosely related to our paper is hyperparameter optimization (a representative example being \cite{li2017hyperband}).
The main difference with our problem is that, besides the standard exploration-exploitation trade-off, here we also have to deal with a trade-off induced by non-stationarity. A hyperparameter optimization algorithm like the one in \cite{li2017hyperband}
adaptively searches in the space of hyperparameters, and the goal is akin to best arm identification. Yet, the feedback is 
assumed to be \textit{stationary}, since
the hyperparameter values do not correspond to stateful objects (as it is the case for our base learners) and hyperparameter configurations are usually evaluated against a separate validation set.
An adversarial variant of the hyperparameter optimization problem was considered in \cite{jamieson2016non}, but their notion of regret is different from ours.

Our problem can be seen as a (rested variant of) the {\em best arm identification problem}, in that our metric reminds the simple-regret that was previously designed for the best-arm identification problem in the standard (stationary) stochastic multi-armed bandit setting (e.g., \cite{audibert2010best,even2006action,gabillon2012best,kaufmann2016complexity}). We recall that best arm identification is aimed at finding the best arm out of a set of alternatives. The problem itself has been investigated from two slightly different viewpoints. In the so called \textit{fixed-confidence} variant, the goal is to minimize the sample complexity (that is, the number of pulls) needed to guarantee that, with some fixed confidence level, the selected arm is truly the one with smallest expected loss. 
In the \textit{fixed-budget} variant, the goal is to find the best-arm within a fixed number of rounds (budget), while minimizing the probability of error.

In our case, we want the learning algorithm to single out with high probability (fixed confidence) the best base learner but, due to the non-stationary nature of the expected loss of base learners, we also want to do so with as few pulls as possible. Hence, we are in a sense combining the two criteria of fixed confidence and fixed budget.

%
%
In the bandits literature, there are two standard ways of modeling non-stationarity: \textit{restless} \citep{Ortner_2014,russac2019weighted,TekinRestedRestless,whittle1988restless} and \textit{rested} \citep{cella2020stochastic,kleinberg2018recharging,kolobov2020online,levine2017rotting,mintz2017non,pike2019recovering,seznec2018rotting} bandits.
In the restless case, the non-stationary nature of 
the feedback is 
determined only by the environment, and the learning policies either try to detect changes in the payoff distribution in order to restart the learning model, or to apply a weight-decay scheme to the collected observations. On the contrary, in the rested model, the non-stationarity depends on the learning policy itself.
For instance, in models like those in \cite{cella2020stochastic,kleinberg2018recharging,kolobov2020online,pike2019recovering}, the expected payoff distribution of one arm is parametrized by the elapsed time since that arm was last pulled. The main leverage given to the proposed solutions is the possibility of observing more unbiased samples corresponding to a fixed arm-delay pair. This simplifies the parameter estimation problem. Similar to the setting we are proposing, in \cite{levine2017rotting,seznec2018rotting} the authors assume the expected loss of an arm to be a monotonically {\em increasing} function of the number of times the arm was pulled. The striking difference is that, in their variant, a simple greedy solution which at each round selects the currently-best arm is actually an optimal solution. Therefore, their learning problem reduces to estimating for each arm the expected loss corresponding to its next pull, and always select the most promising one. In our setting (see Section \ref{Sec:Preliminaries} below), because expected losses are {\em decreasing}, a similar solution would be clearly sub-optimal, since our objective is to identify the arm minimizing the resulting loss at the end of the game.


Finally, a model that shares similarities with our selective training setting is the active model selection problem investigated in \cite{madani2012active}. Again, the authors only investigate stationary scenarios.


%% file: learningSetting_arxiv.tex
\section{Learning setting}\label{Sec:Preliminaries}
%

We consider a set of $K$ arms 
(or learning agents) 
$\K = [K] = \{1,\ldots,K\}$, whose average performance improves as we play them.
At each round $t\in[T]$, 
the learner picks an arm $I_t \in \K$ and observes the realization $X_{I_t,t}$ of a loss random variable whose (conditional) expectation $\mu_{I_t,t}$ is a decreasing function of the number of times arm $I_t$ has been pulled so far. Specifically, for any $i \in \K$ and $t \in [T]$, denote by $\tau(i,t)$ the number of times arm $i$ has been pulled up to time $t$, and by $\F_t$ the $\sigma$-algebra generated by the past history of pulls and loss random variables $I_1,X_{I_1,1},\dots,X_{I_{t-1},t-1}$.
Given a time horizon $T$, a learning policy $\pi$ is a function 
that maps at each time $t \in [T]$ the observed history $I_1,X_{I_1,1}, \ldots, I_{t-1},X_{I_{t-1},t-1}$ to the next action $I_t\in\K$. At the end of round $T$, policy $\pi$ has to commit to (or to output) a given action $\iout \in \K$.
Then, 
we define
%
%
\begin{equation}
\label{e:mu}
\mu_{i,t} \equiv \E_{\F_t}[X_{i,t}] = \frac{\alpha_i}{\big(1+\tau(i,t-1)\big)^\rho} + \beta_i~,
\end{equation}
where exponent $\rho \in (0,1]$ is a known parameter common to all arms while, for all arms $i\in\K$, scaling parameter $\alpha_i$ and position parameter $\beta_i$ are assumed to be non-negative but {\em unknown} to the learning algorithm. We assume $\alpha_i\in[0,U]$ and $\beta_i\in[0,1]$, where the upper extreme $U$ is a known quantity.
Hence, $\mu_{i,t}$ is the expected loss of arm $i$ at round $t$, conditioned on the fact that $i$ has already been played $\tau(i,t-1)$ times during the previous $t-1$ rounds.

As a shorthand, from now on we will use $\mu_{i}(\tau)$ to denote the expected loss of arm $i\in\K$ if pulled so far $\tau\in[T]$ times. 
Notice that when $\alpha_i = 0$ for all $i\in\K$ our setting reduces to the standard stochastic multi-armed bandit setting (e.g. \cite{auer2002finite}).\footnote
{
Observe that the stationary case can equivalently be recovered by setting $\rho = 0$, which is therefore redundant and ruled out by the condition $\rho \in (0,1]$.
} 
It is the decaying component $\frac{\alpha_i}{(1+\tau(i,t-1))^\rho}$ that makes this setting an instance of the {\em rested} bandit setting \cite{cella2020stochastic,kleinberg2018recharging,levine2017rotting,seznec2018rotting}, where the stochastic behavior of the arms depends on the actual policy  $I_1,\ldots, I_{t-1}$ that has so far been deployed during the game.

We compare a learning policy $\pi$ to the optimal policy that knows all parameters ${\{\alpha_i,\beta_i\}}_{i\in \K}$ in advance, and pulls from beginning to end the arm $\iT$ whose expected loss at time $T$ is smallest, i.e.,  
\begin{equation*}
    \iT = \arg\min_{i\in\K}\ \left(\frac{\alpha_i}{T^\rho} + \beta_i\right)~.
\end{equation*}
We define the {\em pseudo regret} of $\pi$ after $T$ rounds as
\begin{equation}\label{Eq:Regret}
    \reg{\pi}(\muu) = \mu_{\iout}\big(\tout\big) - \mu_{\iT}\big(T\big) ~,
\end{equation}
where 
\(
\tout=\tau(\iout,T)
\)
is the random variable counting the number of pulls of arm $\iout\in\K$ after $T$ rounds.
In the above, $\muu\in \{\mu_i\,:\,[T]\to[0,1]\}_{i \in \K }$ collectively denotes
the non-stationary environment generating the observed losses.
Our goal is to bound pseudo-regret $\reg{\pi}(\muu)$ with high probability, where the probability is w.r.t. the random draw of variables $X_{i,t}$ (and possibly the random choice of $I_1,\ldots,I_T$, and $\iout$).

Additionally, we adopt the notion of state $\underline{\tau} = (\tau_1,\tau_2,\dots.\tau_k)\in[T]^K$ to encode the case where, for all $i \in \K$, arm $i$ has been pulled $\tau_i$ times. Notice that when the learning policy is at state $(\tau,\dots,\tau)$, keep sampling all arms in a round-robin fashion (exploring) entails observing $K$ many samples with expected value $\mu_1(\tau), \dots, \mu_K(\tau)$ respectively, and ending up into state $(\tau+1,\dots,\tau+1)\in[T]^K$. Conversely, when the learning policy is at state $(\tau,\dots,\tau)$, then keep pulling the same arm $i\in\K$ for the remaining $T-K\tau$ rounds (exploiting) corresponds to reaching the furthest still reachable state where arm $i$ (which will then be the most pulled one) will have expected loss $\mu_{i}(T-(K-1)\tau)$.

A closer inspection of Eq.~(\ref{Eq:Regret}) reveals that,
unlike standard best-arm identification problems (e.g., \cite{audibert2010best,even2006action,gabillon2012best,kaufmann2016complexity}), our objective here is not limited to predicting which arm is best at the end of the game, but also to pull it as much as we can, that is, to single it out as early as possible.
This also entails that if the arm our policy $\pi$ pulls the most throughout the $T$ rounds is $i \neq \iT$, then it may be better for $\pi$ to output $\iout = i$ rather than $\iT$ itself, even if $\pi$ gets to know at some point the identity of $\iT$ and starts pulling it from that time onward. This is because if, say, for some $t_0$ close to $T$ we have $\tau(i,t_0) = t_0$ and $\tau(\iT,t_0) = 0$, then we may well have $\mu_{i,t_0} < \mu_{\iT,T-t_0}$, so that (\ref{Eq:Regret}) is smaller for $\iout = i$ than for $\iout = \iT$. 
In order to gather further insights, it is also worth considering the simple policy $\pi$ which selects all arms $T/K$ times, and then outputs the best arm $\iT$. According to~(\ref{Eq:Regret}), $\pi$ will still suffer significant regret, since it did not play $\iT$ often enough throughout
the $T$ rounds (that is, $\pi$ has explored \quotes{too much} on sub-optimal arms). We can thus claim that, thanks to the presence of the $\tout$ variable, our regret in (\ref{Eq:Regret}) is only seemingly non-cumulative.
\ignore{
a simple explore-than-commit baseline policy that in an initial phase of length $T_0$ plays each arm $i \in \K$ the same number $T_0/K$ of times in a round-robin fashion, and then commits to the arm that looks best at $T_0$, i.e., $\iout = \arg\min_{i \in \K} \mu_{i,T_0/K}$, where we assume for simplicity that we have direct access to $\mu_{i,T_0}$ for all $i \in \K$. Then we have
\[
\reg = \mu_{i,T(\iout, T)} - \mu_{\iT,T}~.
\]
}

Finally, observe that the average loss $\mu_{i,t}$ in (\ref{e:mu}) can be expressed as the linear combination $\mu_{i,t} = x_t^\top \theta^*_i$, where
$\theta^*_i = [\alpha_i,\beta_i]^\top$ is the unknown vector associated with arm $i$, and $x_t = [1/\tau(i,t-1)^\rho,1]^\top$ is the ``context" vector at time $t$. This might give the impression of some kind of linear contextual bandit (e.g., \cite{soare2014best}) in the best arm identification regime. Yet, this impression is erroneous, since in our problem $x_t$ is itself generated by the learning policy during its online functioning.




%% file: lowerBound_arxiv.tex
\section{Main trade-offs and lower bound}\label{Sec:LowerBound}
In this section we provide a distribution-dependent lower bound for the proposed setting. This will also give us the chance to comment on the specific features of our learning task in terms of the main trade-offs a learning policy has to face.

%

We start by defining the class of {\em arm-elimination policies} as those which periodically remove sub-optimal arms and keep sampling in a round-robin fashion\footnote
{
For simplicity, we restrict here to deterministic policies.
} 
the remaining arms across the rounds. The following simple fact holds.
%
\begin{fact}\label{Fact:PermutationInvariance}
The regret incurred by an arbitrary policy $\pi$ is invariant to permutations of the chronological order of its actions. In fact, \eqref{Eq:Regret} only depends on the arm $\iout$ selected at the end, and the number of times $\tout$ that arm has been chosen during the $T$ rounds. Hence, for any $\pi$, there exists an arm-elimination policy $\pi'$ sharing the same regret (i.e., having the same pair $\iout,\tout$). 

\end{fact}
We can therefore restrict our lower bound investigation to arm-elimination policies.
An advantage of this restriction is a more convenient characterization of the state space $\{\underline{\tau}\}$ associated with the learning problem. The size of the state space is clearly of the form $T^K$. 

Another relevant aspect of our problem 
is that, based on (\ref{e:mu}) and (\ref{Eq:Regret}),
%
for each given state $(\tau_1,\dots,\tau_K)\in[T]^K$, there are at most $K$ many candidate optimal and still reachable states that any policy could end up to. These are specifically the $K$ alternative states that the learning policy at hand would reach by committing to one of the $K$ arms for all the remaining 
rounds. All other states (which are exponential many) can easily be seen to be sub-optimal.

Before moving to the main result of this section (the regret lower bound), we would like to give an additional characterization of the considered class of policies. The missing component which gives a well-specified policy is the condition governing the arm elimination.
Since expected losses (\ref{e:mu}) are non-increasing, and given the regret criterion (\ref{Eq:Regret}), once a policy is confident that sticking to an arm would give a smaller expected loss than the one associated with the last pull,
this policy might be tempted to eliminate all the other arms. In the next example we show that operating this way can be sub-optimal.

\begin{example}\label{Rem:Counterexample}
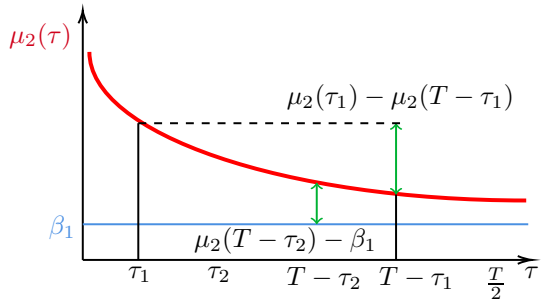
\begin{figure}[t]
    \centering
    \input{remark1_figure}
    \vspace{-0.1in}
    \caption{Expected losses associated with the arms in Example \ref{Rem:Counterexample}.}
    \label{Fig:Counterexample}
\end{figure}
Let us consider the specific instance of our problem with $K=2$ arms whose expected losses are sketched in Figure \ref{Fig:Counterexample}. Whereas the first arm is stationary $\mu_1(\tau) = \beta_1 $, the second is not, $\mu_2(\tau) = \frac{\alpha_2}{\tau^\rho} + \beta_1$. At state $\underline{\tau}_1 = (\tau_1,\tau_1)$ it may occur that the $\tau_1$ observations associated with arm $2$ are enough to realize that $\mu_2(T-\tau_1)<\mu_2(\tau_1)$. 
%
Hence the learning policy knows that if it kept sampling arm $2$ for the remaining $T-2\tau_1$ pulls it would achieve a smaller (expected) loss compared to $\mu_2(\tau_1)$. The same would not hold for the other arm, as it is stationary.

Let us now denote by $\tau_2$ the number of pulls it takes to figure out that $\beta_1<\mu_2(T-\tau_2)$. It could be the case that $\tau_2 > \tau_1$ (that is, as in Figure \ref{Fig:Counterexample}, we have $\mu_2(T-\tau_2)-\beta_1 < \mu_2(\tau_1) - \mu_2(T - \tau_1)$). In order to maximize $\tout$ (so as to minimize regret (\ref{Eq:Regret})) a naive policy might eliminate arm $1$ after $\tau_1$ observations. This would translate into choosing the wrong value of $\iout$, hence clearly incurring a regret. Conversely, a smarter policy that keeps exploring up to state $(\tau_2,\tau_2)$ would return $\iout = 1$ and yield $\tout = T - \tau_2$. Notice that, thanks to the stationary nature of the optimal arm, the regret incurred by this smarter policy is indeed zero.
\end{example} 
%
All in all, the above observations help better understand the structure of our problem, which will be useful in all technical proofs (see the appendix).

We can now turn our attention to the lower bound. In doing so, we generalize the results in \cite{bubeck2013bounded}, which in turn adopts a hypothesis testing argument that hinges on a lower bound for the minimax risk of hypothesis testing (e.g., \cite{tsybakov2008introduction}, Chapter 2). Notice that the classical lower bound result for stationary stochastic bandits \cite{lai1985asymptotically} cannot easily be adapted here since, being asymptotic in nature, that result tends to lose the non-stationary component of our expected losses (\ref{e:mu}), and thus the cumulated effect of this non-stationarity on the $\tout$ variable.




As done in \cite{bubeck2013bounded}, for all arms and all possible number of pulls, we consider all families of loss distributions $\{\Pro_\mu\}$, indexed by their expected value $\mu$, and such that $KL(\Pro_{\mu}, \Pro_{\mu'}) = C (\mu - \mu')^2$ for some absolute constant $C>0$ (e.g., in the case of normal distributions, $KL(\mathcal{N}(\mu,\sigma), \mathcal{N}(\mu',\sigma))=\frac{1}{2}(\mu-\mu')^2$).


In the sequel, we use $\tsub=T-\tout$ to denote the number of rounds spent by pulling all arms different from $\iout$. Additionally, we denote by $\Pro_{\muu(\tau)}=\Pro_{\mu_1(\tau)}\otimes\dots\otimes\Pro_{\mu_K(\tau)}$ the product distribution that generates the losses from $\Pro_{\mu_i(\tau)}$ when pulling arm $i\in\K$ for the $\tau$-th time.
The result that follows restricts to the two arm case,\footnote
{
We believe that restricting to the two arm case helps better elucidate the nature and trade-offs in our problem. We conjecture that a similar but considerably more involved result can be shown for $K$ arms.
}
and delivers a bound on the regret that holds in expectation over the random draw of the losses.

%
\begin{theorem}\label{Th:RLB} Let $\Pro_{\muu(\tau)} = \Pro_{\mu_1(\tau)}\otimes \Pro_{\mu_2(\tau)}$ be defined by distributions whose expected values are $\mu_1(\tau) = \frac{\ab}{\tau^\rho} + \beta$ and $\mu_2(\tau) = \mu_1(\tau) + \Delta$, respectively,
where $\Delta > 0$ is an unknown but fixed constant. Additionally, let $\Pro_{\muu'(\tau)} = \Pro_{\mu'_1(\tau)}\otimes \Pro_{\mu'_2(\tau)}$ be another product distribution, whose expected value components are $\mu'_1(\tau) = \mu_1(\tau)$ and $\mu'_2(\tau) = \mu_1(\tau) - \Delta$. Then, for any policy $\pi$, and any horizon $T\geq 1$, the following holds:
\begin{equation*}
        \max\Big\{\E\big[\reg{\pi}(\muu)\big], \E\big[\reg{\pi}(\muu')\big]\Big\} \geq \ab \bigg( \frac{1}{(T-\tsub)^\rho} - \frac{1}{T^\rho}\bigg)~,
\end{equation*}
where $\tsub$ is the smallest $\tau \in [T]$ which is strictly larger than
    \[
    \min\bigg\{\frac{T}{2},\, \frac{ \log T}{C\Delta^2},\, \frac{\log T}{C\ab^2}(T-\tau-1)^{2\rho+2}\bigg\}~.
    \]
\end{theorem}
%
In the proof (see 
the appendix), one can also find the exact expression for $\tsub$, which is slightly more complex that the one given above. 
The main idea behind the proof is that in the considered learning problem, we have two quantities characterizing the lower bound on the number of sub-optimal pulls $\tsub$. The first one is associated with $\iout$, and is of the order of $\frac{1}{\Delta^2}$. The second one is induced by the objective of minimizing the incurred expected loss at $\tout$. 
The main point here is that exploring towards arm $\iout$ is worthwhile only if it does not cause a higher incurred loss $\mu_{\iout}(\tout)$.

In the proof, we first show a lower bound on $\tsub$ of the form $1/\Delta^2$ even if all parameters $\ab,\beta,\Delta$ are known (as in the standard stationary case). Then, we show that the sole knowledge of $\Delta$ (already) causes a rescaling of the lower bound of order $\log T$. 


We would like to emphasize that, if we let $\Delta\to0$ (that is, the two arms are less and less statistically distinguishable), and consider a large enough time horizon $T$, an optimal strategy for our regret minimization problem is by no means to pull both arms an equal ($T/2$) number of times. Rather, an optimal strategy would commit to one of the two arms as soon as it is confident enough on which of them has the smaller loss (at any reachable state), unless trying to determine the best arm causes a bigger regret than immediately committing to any of the two.


The quantity $\tsub$ (at least in the two-arm case) will play a central role in characterizing the statistical complexity of our learning problem.

A better interpretation of the bound contained in Theorem \ref{Th:RLB} is provided by the below corollary, 
where we set as a relevant example $\rho= 1/2$. 
%
%
\begin{corollary}\label{Co:RLB}
Let the same assumptions as in Theorem \ref{Th:RLB} hold with 
$\rho=1/2$. Then, for any policy $\pi$, and any horizon $T\geq 1$, the following holds:
\vspace{-0.1in}
\begin{equation*}
        \max\Big\{\E\big[\reg{\pi}(\muu)\big], \E\big[\reg{\pi}(\muu')\big]\Big\} \geq \ab \bigg( \frac{1}{\sqrt{T-\tsub}} - \frac{1}{\sqrt{T}}\bigg)~,
\end{equation*}
where $\tsub$ is the smallest $\tau\in[T]$ which is strictly greater than
\begin{align*}
\tsub = \min\bigg\{\frac{T}{2},\ceil*{\frac{\log T}{C \Delta^2}},\frac{\log T}{C\ab^2}(T-\tau-1)^{3} \bigg\}~. 
\end{align*}
\end{corollary}
%
%
%
%
These results help elucidate the novel trade-off characterizing the proposed non-stationary bandit problem. In fact, the variable $\tsub$, which can be interpreted as counting the number of pulls of the sub-optimal arm, is not only a function of $1/\Delta^2$, as is for the stationary case. Here, $\tsub$ also depends on the relative size of $T$, $\ab$ and $1/\Delta^2$.
The main intuition behind this result is that, the more the samples, the better a learning policy $\pi$ may understand the shape of the arms' expected loss functions. 
In particular, at a given round, $\pi$ need not realize which arm is optimal (the arms may or may not be statistically equivalent), still $\pi$ might realize that, the more it keeps exploring the higher the expected loss it incurs at $\tout$.

%% file: remark1_figure.tex
\tikzset{every picture/.style={line width=0.75pt}} 
\begin{tikzpicture}[x=0.5pt,y=0.5pt,yscale=-0.9,xscale=1]

\draw    (48,220) -- (386,220) ;
\draw [shift={(388,220)}, rotate = 180] [color={rgb, 255:red, 0; green, 0; blue, 0 }  ][line width=0.75]    (10.93,-3.29) .. controls (6.95,-1.4) and (3.31,-0.3) .. (0,0) .. controls (3.31,0.3) and (6.95,1.4) .. (10.93,3.29)   ;
\draw    (48,220) -- (48,12) ;
\draw [shift={(48,10)}, rotate = 450] [color={rgb, 255:red, 0; green, 0; blue, 0 }  ][line width=0.75]    (10.93,-3.29) .. controls (6.95,-1.4) and (3.31,-0.3) .. (0,0) .. controls (3.31,0.3) and (6.95,1.4) .. (10.93,3.29)   ;
\draw [color={rgb, 255:red, 74; green, 144; blue, 226 }  ,draw opacity=1 ]   (48,190) -- (385,190) ;
\draw  [draw opacity=0][line width=1.5]  (383.46,170) .. controls (383.31,170) and (383.15,170) .. (383,170) .. controls (200.75,170) and (53,114.04) .. (53,45) -- (383,45) -- cycle ; \draw  [color={rgb, 255:red, 255; green, 0; blue, 0 }  ,draw opacity=1 ][line width=1.5]  (383.46,170) .. controls (383.31,170) and (383.15,170) .. (383,170) .. controls (200.75,170) and (53,114.04) .. (53,45) ;

\draw [<->, mycolor]   (285,165) -- (285,105) ; 
\draw  (285,165) -- (285,220) ; 
\draw    (90,220) -- (90,105) ; 
\draw  [<->, mycolor]  (225,155) -- (225,190); 
\draw  [dashed]  (90,105) -- (290,105) ; 

\draw (-10,18) node [anchor=north west][inner sep=0.75pt]  [color={rgb, 255:red, 208; green, 2; blue, 27 }  ,opacity=1 ]  {$\mu_{2}( \tau )$};
\draw (20,180) node [anchor=north west][inner sep=0.75pt]    {$\textcolor[rgb]{0.29,0.56,0.89}{\beta_1}$};
\draw (80,223.4) node [anchor=north west][inner sep=0.75pt]    {$\tau_1 $};
\draw (140,223.4) node [anchor=north west][inner sep=0.75pt]    {$\tau_2 $};
\draw (200,223.4) node [anchor=north west][inner sep=0.75pt]    {$T - \tau_2 $};
\draw (270,223.4) node [anchor=north west][inner sep=0.75pt]    {$T - \tau_1 $};
\draw (350,223.4) node [anchor=north west][inner sep=0.75pt]    {$\frac{T}{2}$};
\draw (380,223.4) node [anchor=north west][inner sep=0.75pt]    {$\tau $};
\draw (130,190) node [anchor=north west][inner sep=0.75pt] {$\mu_2(T - \tau_2) - \beta_1 $};
\draw (200,70) node [anchor=north west][inner sep=0.75pt] {$\mu_2(\tau_1) - \mu_2(T - \tau_1)$};
\end{tikzpicture}

%% file: bernsteinEstimators_arxiv.tex
\section{
Estimation of parameters}\label{Sec:ParamEst}
In order to minimize the regret $\reg{\pi}(\muu)$ any reasonable policy $\pi$ has to be able to estimate, for all arms $i \in \K$, the associated expected loss $\mu_i(\cdot)$, and it has to do so at any still reachable state where arm $i$ will be pulled $\tout$ times. To this effect, we now introduce two 
statistically independent estimators. 
Upon pulling arm $i\in\K$ for $2\tau$ times, we define
\begin{align}\label{Eq:Estimators}
    &\Xh_{i,\tau} = \frac{1}{\tau} \sum_{s=1}^{\tau} X_i(s)~,
    &\Xt_{i,\tau} = \frac{1}{\tau} \sum_{s= \tau+1}^{2\tau} X_i(s)~,
\end{align}
where $X_{i,s}$ denotes the loss incurred by arm $i$ after having pulled it $s$ times ($\E[X_{i,s}]=\mu_{i,s}$). Notice, that due to the way we have defined $\mu_{i,s}$,  
the above estimators are empirical averages of {\em independent} but non-identically distributed random variables, the independence deriving from the fact that pulling one arm does not influence the distribution of the others. Moreover, because of the time decay, the expectation of $\Xh_{i,\tau}$ cannot be smaller than the expectation of $\Xt_{i,\tau}$.

We combine these two estimators together with standard concentration inequalities to
derive a joint estimator for $(\alpha_i,\beta_i)$. Since the two estimators are non-redundant, this allows us to come up with estimators for $\alpha_i$ and $\beta_i$ individually. 


Using Bernstein's inequality,\footnote
{
It is worth mentioning in passing that the standard Hoeffding inequality delivers vacuous estimators, as the range of the random variables $X_{i,t}$ alone does not carry enough information about the concentration properties of $\Xh_{i,\tau}$ and $\Xt_{i,\tau}$.  
} 
we can derive confidence bounds around $\Xh_{i,\tau}$ and $\Xt_{i,\tau}$ as functions of $\beta_i$ and $\alpha_i$. Specifically, for each arm $i\in\K$, number of pulls $2\tau\in[T]$, the expectation $\E[\Xh_{i,\tau}]$ is contained with probability at least $1-\delta$ in the interval $[\Xh_{i,\tau}- \text{CB}_{\Xh,\tau}(\delta), \Xh_{i,\tau} + \text{CB}_{\Xh,\tau}(\delta)]$ where
\begin{equation*}
    \text{CB}_{\Xh,\tau}(\delta) = \left(\sqrt{U} + 1 \right)^2 \sqrt{\frac{2}{\tau}\log\frac{1}{\delta}} + \frac{(U+1)\log\frac{1}{\delta}}{\tau}~. 
\end{equation*}
Likewise, 
$\E[\Xt_{i,\tau}]$ is contained with probability at least $1-\delta$ in the interval $[\Xt_{i,\tau}- \text{CB}_{\Xt,\tau}(\delta), \Xt_{i,\tau} + \text{CB}_{\Xt,\tau}(\delta)]$ where
\begin{equation*}
    \text{CB}_{\Xt,\tau}(\delta) =  \left(\sqrt{U} + 1 \right)^2 \sqrt{\frac{2}{\tau}\log\frac{1}{\delta}} + \frac{(U+1)\log\frac{1}{\delta}}{\tau}~. 
\end{equation*}
Starting from these definitions we can build the following set of inequalities
%
\begin{align*}
    \E[\Xh_{i,\tau} ] - \text{CB}_{\Xh,\tau}(\delta)  \leq \Xh_{i,\tau} \leq 
\E[\Xh_{i,\tau} ] + \text{CB}_{\Xh,\tau}(\delta) \\
    \E[\Xt_{i,\tau} ] - \text{CB}_{\Xt,\tau}(\delta)  \leq \Xt_{i,\tau} \leq 
\E[\Xt_{i,\tau} ] + \text{CB}_{\Xt,\tau}(\delta) 
\end{align*}
%
which can be solved for $\alpha_i$ and $\beta_i$ individually. As shown in the appendix, 
this gives rise to the following confidence intervals for $\alpha_i$ :
\begin{align}\label{Eq:alpha}
    \alpha_i\, &\in\,\, \overbrace{\frac{\tau \DX_{i,\tau} }{\sum_{s=1}^\tau \frac{1}{s^\rho} - \sum_{s=\tau+1}^{2\tau} \frac{1}{s^\rho}}}^{\ah_{i,\tau}} \pm \frac{5\tau^\rho\bigl(\sqrt{U}+1\bigl)^2}{\rho} \Bigl[  \frac{\log1/\delta}{\tau} 
    + \sqrt{\frac{1}{\tau} \log\frac{1}{\delta}}\Bigl]~,
\end{align}
where $\DX_{i,\tau}=\Xt_{i,\tau} - \Xh_{i,\tau}$. For brevity, the confidence interval centroid will be denoted by $\ah_{i,\tau}$. 
Similarly, $\beta_i$ can be shown to
satisfy
\begin{align}\label{Eq:beta}
    \beta_{i}\, &\in\,\, \overbrace{\Xh_{i,\tau} - \frac{\ah_{i,\tau}}{\tau} \sum_{s=1}^\tau \frac{1}{s^\rho}}^{\bh_{i,\tau}} \pm \frac{5\bigl(\sqrt{U}+1\bigl)^2}{(1-\rho)\rho} \bigg[ \frac{\log1/\delta}{\tau} +  \sqrt{\frac{1}{\tau}\log\frac{1}{\delta}}\bigg] ~,
\end{align}
\noindent where $\bh_{i, \tau}$ denotes the centroid of confidence interval (\ref{Eq:beta}).
Despite we have provided separate estimators for $\alpha_i$ and $\beta_i$, it is important to stress that our interest here is not to  estimate them separately. We combine these estimators to compute
\vspace{-0.07in}
\[
\muh_{i,\tau}(\tout) = \frac{\ah_{i,\tau}}{\tout^\rho} + \bh_{i,\tau}~,
\]
an estimate of the expected loss incurred by arm $i\in\K$ as if we had pulled it $\tout$ times after having observed only $2\tau$ realizations of $X_{i,t}$. 
All the above can be summarized by the following theorem.
\begin{theorem}\label{Th:LossEstimate}
After observing $X_{i,1},\dots,X_{i,2\tau}$ loss realizations of arm $i\in\K$, 
we can predict the expected loss $\mu_{i,\tout}$ of arm $i$ as it were pulled $\tout$-many times (with $\tout>\tau)$. In particular, we have that with probability at least $1-\delta$ jointly over $i \in \K$, $\tau \in [T]$ and $\tout \in [T]$
\[
\muh_{i,\tau}(\tout) - \text{CB}_{\mu,\tau}(\delta) \leq \mu_{i,\tout} \leq \muh_{i,\tau}(\tout) + \text{CB}_{\mu,\tau}(\delta)
\]
where
\begin{align}
        \text{CB}_{\mu,\tau}(\delta) = \frac{10\left(\sqrt{U}+1\right)^2}{(1-\rho)\rho}  \Bigg[ \frac{\log \frac{\tau KT}{\delta}}{\tau} + \sqrt{\frac{1}{\tau}\log\frac{\tau KT}{\delta}}\Bigg]~. \nonumber
    \end{align}
\end{theorem}
%
%
Hence, the approach contained in Theorem 
\ref{Th:LossEstimate} allows us to obtain confidence intervals for $\mu_{i,\tout}$ shrinking with $\tau$ as $\frac{1}{\sqrt{\tau}}$ up to a numerical constant depending on $\rho$ and $U$. 

Finally, observe that these confidence intervals are non-vacuous only when $\rho\in(0,1)$, that is, excluding the extreme cases $\rho = 0$ and $\rho = 1$. The case $\rho=0$ is indeed uninteresting, since it yields a stationary case which is equivalent to the one achieved by the setting $\alpha_i = 0$ for all $i$. In fact, due to the specific nature of the empirical averages in (\ref{Eq:Estimators}),
when $\rho=0$ the centroid $\ah_{i,\tau}$ occurring in (\ref{Eq:alpha}) is not well defined, independent of the number of observed samples $\tau$. 
On the other hand, because our derivations rely on approximations of the form $\sum_{s=1}^\tau\frac{1}{s^\rho} \approx \frac{s^{1-\rho}}{1-\rho}$, which only hold for $\rho\neq1$, the case $\rho = 1$ should be treated separately via standard approximations of the form $\sum_{s=1}^\tau\frac{1}{s} \approx \log \tau$. We leave this special case to the full version of the paper.

The above estimators will be the building blocks of our learning algorithms, presented in the next section.
In particular, the definition of $CB_{\mu,\tau}(\delta)$ provided in Theorem \ref{Th:LossEstimate} above will be repeatedly used throughout the rest of the paper.

%% file: etc_arxiv.tex
\section{Regret minimization}\label{Sec:LearningProb}
In this section we present two learning policies.
We first describe as a warm-up a simple explore-then-commit strategy, then we present a more sophisticated strategy inspired by the Successive Reject algorithm \cite{audibert2010best}. For both policies, we set the confidence parameter $\delta$ to $\frac{1}{T}$.

The first solution we propose is a rested bandit variant of the standard explore-then-commit (ETC) policy (e.g., Ch. 6 in \cite{lattimore2018bandit}). 
In its original formulation, ETC requires as input a parameter $n\in[T]$ specifying the number of initial pulls associated with each arm. Once all the arms have been pulled $n$ times, the  exploratory stage finishes. The original ETC algorithm then sticks to the most promising arm according to the estimates computed during exploration.
Hence the two phases of exploration and exploitation are kept separate.
This strategy has a clear limitation.
Since the exploration parameter $n$ is an input to the algorithm,
the original ETC algorithm does not adapt the length of the exploration phase to the actual samples, so that understanding how to best set $n$ 
is not a simple task.

One thing that is worth noticing is that in the 2-arm bandit case, this parameter $n$ takes values in the range $[T/2]$. If $\tsub$ in our lower bound of Theorem \ref{Th:RLB} equals $T/2$ (that is, when $T/2$ is smaller than both $\frac{1}{\Delta^2}\log T$ and $\arg\min\{{n\in[T]\,:\, n>\frac{\log T}{C\ab^2}(T-n-1)^{2\rho+2}}\}$), we cannot commit to any specific arm, and the ETC algorithm results in a solo-exploration strategy which is indeed optimal in this case.
\begin{algorithm}[t]
\caption{Explore-Then-Commit (ETC)}
\label{Alg:FC_ETC}
\begin{algorithmic}[1]
\REQUIRE Confidence parameter $\delta = 1/T$
\FOR{$n \in 1,\dots,\floor{T/K}$}
\STATE pull each arm once
\STATE $\tout = T - n (K-1)$
\IF{$\exists i \in \K: \; \muh_{i,n}(\tout) < \min_{j\in\K\setminus\{i\}} \muh_{j,n}(\tout) - 2 CB_{\mu,n}(\delta)$}\label{AlgLine:ETCCondition}
\STATE $\iout = \arg\min_{i\in\K}\muh_{i,n}(\tout)$
\STATE break; 
\COMMENT{The exploration phase terminates}
\ENDIF
\ENDFOR
\STATE Play $\iout$ until round $T$ \COMMENT{Commit}
\label{AlgLine:ETCCommit}
\STATE Output $\iout$
\end{algorithmic}
\end{algorithm}

Algorithm \ref{Alg:FC_ETC} describes a variant of
the standard ETC policy adapted to our rested bandit scenario.
At a generic round $t=Kn$, this algorithm starts committing to an arm $i\in\K$ only when we are confident with probability at least $1-\delta$ that $i$ is the arm with lowest expected loss if pulled for the remaining $T-Kn$ times (Line \ref{AlgLine:ETCCondition}). Hence, unlike the original ETC algorithm, this algorithm implicitly computes $n$ on the fly based on the observed samples. Finally, upon committing to an arm, our algorithm does not reconsider its decision based on the newly collected samples (Line \ref{AlgLine:ETCCommit}).

We have the following result, that help elucidate the benefit of adaptively inferring $n$. 
%
%
\begin{theorem}\label{Th:FC_ETC}
Consider the same two-arm setting $\Pro_{\muu(\tau)} = \Pro_{\mu_1(\tau)}\otimes \Pro_{\mu_2(\tau)}$ contained in Theorem \ref{Th:RLB} and the notation introduced therein.
Running Algorithm \ref{Alg:FC_ETC} with $T\geq 1$ 
yields
\begin{equation*}
    \reg{ETC}(\muu) \leq \ab\bigg( \frac{1}{(T-n_0)^\rho} - \frac{1}{T^\rho} \bigg) + {\tilde O}\bigg(\frac{1}{\sqrt{T}}\bigg)~
\end{equation*}
with probability at least $1-\frac{1}{T}$ where $n_0 = \min\Big\{\frac{T}{2}, \frac{c_\rho}{\Delta^2}\Big\}$, $c_\rho=\frac{1600(\sqrt{U}+1)^4}{\rho^2(1-\rho)^2}\log(4n_0 T^2)$ and ${\tilde O}(\cdot)$ hides $\log T$ factors. This result is optimal up to a logarithmic factor whenever $\tsub$ in Theorem \ref{Th:RLB} is not equal to $\arg\min\{{\tau\in[T]\,:\, \tau>\frac{\log T}{C\ab^2}(T-\tau-1)^{2\rho+2}}\}$, but instead corresponds to the minimum between $\frac{T}{2}$ and $\frac{\log T}{C \Delta^2}$.
\end{theorem}
%


%
%
Starting from ETC, in the next section, we present our final learning policy, which will be analyzed in the general $K$-armed case.

%% file: cl_arxiv.tex
\subsection{Towards an Optimal Policy}
The first limitation of the ETC strategy in Algorithm \ref{Alg:FC_ETC} becomes clear when considering more than $2$ arms. Let us consider an instance with $K=3$ arms where there exist two values $n_2,n_3$ satisfying:
\begin{align*}
    \muh_1(n') &< \muh_2(n') - 2 CB_{\mu,n_2}(n_2) \hspace{1em} \forall n'>n_2\\
    \muh_1(n') &< \muh_3(n') - 2 CB_{\mu,n_3}(n_3) \hspace{1em} \forall n'>n_3~.
\end{align*}
The ETC policy in Algorithm \ref{Alg:FC_ETC} has a single counter $n$ that has to satisfy at the same time $K-1=2$ arm elimination conditions (line \ref{AlgLine:ETCCondition} of Algorithm \ref{Alg:FC_ETC}). The best this algorithm can do in order not to commit to the wrong arm is to keep exploring up to $n=\max\{n_2,n_3\}$. The obvious drawback of this solution is that ETC would then waste $|n_2-n_3|$ pulls on the sub-optimal arms $2$ and $3$, rather than selecting $\iout=1$.
\begin{algorithm}[t]
\caption{\RS}
\label{Alg:RESTSURE}
\begin{algorithmic}[1]
\REQUIRE Confidence parameter $\delta = 1/T$
\STATE Initialize: $\A_0 = \K, n=0, \tout = T$, and $t=0$.
\FOR{$t=1,\dots,T$}
\STATE $\tout = T - t + n$
\IF{$\exists i \in \A_n: \; \muh_i(\tout) < \min_{j\in\A_{n}\setminus\{i\}}\muh_j(\tout) - 2 CB_{\mu,n}(\delta) \;$} \label{AlgLine:Commit}
\STATE break;
\COMMENT{Found $i^*_{\tout}$ w.h.p.}
\ENDIF
\IF{$ \min_{i\in\A_n} \muh_{i,n}(\tout-|\A_n| + 1) - 2 CB_{\mu,n}(\delta) > \min_{i\in\A_n} \muh_{i,n}(\tout)$} \label{AlgLine:StopExploration}
\STATE break;
\COMMENT{No advantage in learning $i^*_{\tout}$}
\ENDIF
\STATE $\A_{n+1} = \A_n$ 
\FOR{\textbf{each} arm $i\in\A_{n+1}$ such that $\forall m \in [n, \tout]: \exists j \in \A_{n+1}: \muh_{i,n}(m) -  \muh_{j,n}(m) > 2 CB_{\mu,n}(\delta)$}\label{AlgLine:DAEElimination}
\STATE $\A_{n+1} = \A_{n+1} \setminus \{i\}$ \COMMENT{Arm Elimination}
\ENDFOR
\STATE Pull once each active arm $i\in\A_{n+1}$
\STATE $t = t + |\A_{n+1}|$;\quad $n = n + 1$
\ENDFOR
\STATE Play $\iout$ until round $T$ \COMMENT{Commit} \label{AlgLine:DAECommit}
\STATE Output $\iout$
\end{algorithmic}
\end{algorithm}

We now present in Algorithm \ref{Alg:RESTSURE} the strategy \RS\ (RESTed SUccessive REject), a rested version of the Successive Reject algorithm from \cite{even2006action,audibert2010best}. As for its stationary counterpart, \RS keeps sampling all the active arms in a round-robin fashion, and then periodically removes arms once it is confident about their sub-optimality (line \ref{AlgLine:DAEElimination}). The key adaptation to our rested bandit scenario is that one arm is deemed sub-optimal when there is a better arm in {\em any} of the still reachable states. 

Going into some details of the pseudocode, the stopping condition in lines \ref{AlgLine:Commit}-5 of Algorithm \ref{Alg:RESTSURE} is inspired by the same reasoning governing the commitment in the stationary bandit problem. This condition tells us that exploration has provided enough information to identify (with high probability) arm $\iout=\arg\min_{i\in\K}\mu_i(\tout)$ at the best reachable state. The second stopping condition (lines \ref{AlgLine:StopExploration}-8) is due to the non-stationary component in the expected loss (\ref{e:mu}). This condition controls the trade-off between the estimation of $\iout=\arg\min_{i\in\K}\mu_i(\tout)$ and the minimization of the incurred expected loss, namely the impact on the value of $\tout$. In particular, this condition stops the policy in its exploration towards the identity of $\iout$ as soon as this would cause an increased regret due to a reduced valued of $\tout$.

We need the following additional notation. 
We set for brevity $\Delta_{j,i}(\tau) = \mu_j(\tau) - \mu_{i}(\tau)$ for any $\tau\in[T]$, $K_{n} = K - n$, 
and $\mu^*(\tau) = \min_{i\in\K}\mu(\tau)$ denotes the smallest expected loss over all arms after each one of them has been pulled exactly $\tau$ times.
The following is the main result of this section.
%
\begin{theorem}\label{Th:RESTSURE_RB} For all $K>1$, if \RS is run on $K$ arms having arbitrary non-stationary loss distributions $\Pro_{\muu(m)}=\Pro_{\mu_1(m)},\otimes,\dots,\otimes,\Pro_{\mu_K(m)}$ with support in $[0,1]$ and expected value parameterized according to (\ref{e:mu}), then with probability at least $1-\frac{1}{T}$ the pseudo-regret of \RS after $T$ interactions satisfies 
\begin{equation*}
    \reg{\RS}(\muu) \leq \mu_{\iout}(T-\nb) - \mu_{\iT}(T),
\end{equation*}
where $\nb = \sum_{s\in[K-1]}n_{\sigma(s)}$, and 
$n_{\sigma(s)}$ is defined as the smallest $n\in[T]$ which is greater than
\begin{align*}
    \min\Bigg\{ &\frac{T-\sum_{j=1}^{s-1}n_{\sigma(j)}}{K_{s-1}},\\[-4mm]
    &\quad\frac{c_{\rho}\log (nK^2T^2)}{\min_{j\in\As_{s-1},m\in[n_{\sigma(s)},\tout(s)]}\Delta^2_{\sigma(s),j}(m)},\\[-0mm]
    &\quad\frac{c_{\rho}\log (nK^2T^2)}{\big(\mu^*\big(\tout(s)-K_{s+1}\big) -\mu^*\big(\tout(s)\big)\big)^2},\\[-0mm]
    &\quad\frac{c_{\rho}\log (nK^2T^2)}{\big(\min_{j\in\K} \big( \Delta_{\sigma(s),j}(\tout(s)) \big)^2} \Bigg\}.
\end{align*}
%
In the above,
$\As_s = \K \setminus \{\sigma(1),\ldots, \sigma(s-1)\}$,
$\tout(n) = T -\sum_{s=1}^{n} K_{s+1} n_{\sigma(s)}$ and $c_{\rho}=\frac{1600(\sqrt{U}+1)^2}{\rho^2(1-\rho)^2}$.
Notice that $\sigma(s) = \arg\min_{j\in\As_{s-1}}n_{\sigma(j)}$.
Finally, $\iout\in\arg\min_{i\in\K}\mu_i(T-\nb)$ only if 
\vspace{-0.05in}
\[
\min_{j\in\As_{s-1}}\Delta_{\sigma(s),j}(\tout(s)) > \mu^*\big(\tout(s)-K_{s+1}\big) -\mu^*\big(\tout(s)\big)~.
\]
Conversely, when the latter condition is not met we can only guarantee that 
\[
\mu_{\iout}(T-\nb)\leq\mu^*(T-\nb) + 2 CB_{\mu,n_{\iout}}(1/T)~.
\]
\end{theorem} 
Notice that $\nb$ is solely a function of the problem parameters $\{\alpha_i,\beta_i\}_{i\in \K}, \rho, K$, $U$, and $T$. This is because so are the involved quantities $\sigma(s)$ and $n_{\sigma(s)}$.

The exact expression for $\nb$ might be somewhat hard to interpret.
The first term in the $\min$ 
plays the same role as term $T/2$ in Theorem \ref{Th:RLB}, and guarantees the total number of pulls is most $T$. The second term 
is obtained from the arm-elimination condition of line \ref{AlgLine:DAEElimination}.
The third term in the $\min$
is obtained by analyzing the condition at lines \ref{AlgLine:StopExploration}-8 in Alg. \ref{Alg:RESTSURE}. Finally, the commitment to arm $\iout$ yields the fourth term. 
The proof of this theorem 
can be seen as an extension of the proof of the regret bound of Algorithm \ref{Alg:FC_ETC} to the more general scenario considered here.


For the sake of comparison, the next corollary contains a simpler result that specifically applies to the restricted setting of Theorem \ref{Th:RLB} and Theorem \ref{Th:FC_ETC}. This is also meant to demostrate, in this restricted setting, the optimality of the regret bound of \RS.
%
%
\begin{corollary}\label{Co:RESTSURE_Bound} Let us consider the same two-arm setting $\Pro_{\muu(\tau)}=\Pro_{\mu_1(\tau)}\otimes\Pro_{\mu_2(\tau)}$ as in Theorem \ref{Th:FC_ETC}. Then, running Algorithm \ref{Alg:RESTSURE} with $T\geq 1$ yields
\[
    \reg{\RS}(\muu) \leq \ab \bigg(\frac{1}{\sqrt{T-n_0}} - \frac{1}{\sqrt{T}} \bigg) + {\tilde O}\Bigl(\frac{1}{\sqrt{T}}\Bigl)~,
\]
with probability at least $1-\frac{1}{T}$, where $n_0$ is the smallest $n\in[T]$ which is greater than
\begin{align*}
    n_0 = \min\bigg\{\frac{c_\rho}{\Delta^2}, c_\rho \frac{(T-n_0)^3}{\ab^2}\bigg\}~,
\end{align*}
$c_{\rho}=25600(\sqrt{U}+1)^4\log(4 n_0 T^2)$ and ${\tilde O}(\cdot)$ hides $\log T$ factors.
\end{corollary}
\noindent Notice that the upper bound in Corollary \ref{Co:RESTSURE_Bound} matches up to log factors the lower bound given in Corollary \ref{Co:RLB}. 
%
%

%% file: conclusions.tex
\vspace{-0.12in}
\section{Conclusions and ongoing research}
\vspace{-0.1in}
In this work we have proposed an online model selection problem formulated as a best arm identification within a specific rested bandit scenario.
Here, each arm represents a candidate learning model and each pull corresponds to giving the associated learner more i.i.d training samples, thus allowing the learner to reduce its generalization error. 
We formulated an ad hoc notion of regret, provided a lower bound for the learning problem, and analyzed two alternative strategies, one of which we have shown to be optimal in the special cases covered by the lower bound.
%
We are currently trying to extend our lower bound technique to cover our task in substantially broader generality than what is currently contained in Theorem \ref{Th:RLB}.

This was mainly a theoretical work. However, we are planning on conducting an experimental analysis so as to both corroborate our theoretical findings and compare to  methods available in the literature, adapted to our framework.

%% file: appendix_arxiv.tex
%
%
This appendix 
provides the proofs of all theorems and corollaries contained in the main body of the paper. The presentation is split into sections corresponding to the section of the main body.

%
\section{Proofs for Section \ref{Sec:LowerBound}}
\subsection{Proof of Theorem \ref{Th:RLB}}\label{SuppSec:ThRLB}
Inspired by \cite{bubeck2013bounded}, our proof rephrases our bandit learning task (Section \ref{Sec:Preliminaries}) as a hypothesis testing problem, and relies on the following well-known lower bound result for the minimax risk of hypothesis testing (see, e.g., Chapter 2 of \cite{tsybakov2008introduction}).
\begin{lemma}\label{Le:Tsyabkov}
    Let $\Pro_{\mu_1},\Pro_{\mu_2}$ be two probability distributions supported on some set $\X$, having expected value $\mu_1,\mu_2$, and let $\Pro_{\mu_2}$ be absolutely continuous w.r.t. $\Pro_{\mu_1}$. Then for any measurable function $\pi:\X\rightarrow\{1,2\}$, we have
    \begin{align*}
        \Pro_{Y\sim\Pro_{\mu_1}}&\big(\pi(Y) = 2 \big) +\Pro_{Y\sim\Pro_{\mu_2}}\big(\pi(Y) = 1 \big) \geq
        \frac{1}{2} \exp\big(- KL\big(\Pro_{\mu_1}, \Pro_{\mu_2}\big)\big)~. 
    \end{align*}
\end{lemma}
Given the above, let us turn our attention to the proof of Theorem \ref{Th:RLB}.
First, observe that in the considered instances $\muu,\muu'$, one arm always outperforms the other, independent on the number of pulls (i.e., $\mu_1(\tau) < \mu_2(\tau) \; \forall \tau \in [T]$). Additionally, since the considered expected loss (\ref{e:mu}) decreases with $\tout$, in order to obtain a lower bound it suffices to upper bound $\tout = T - \tau(2,T)$, and set $\iout=1$. According to Fact \ref{Fact:PermutationInvariance}, we will only consider arm-elimination policies that after each round-robin phase consider a possible arm removal.\\
As explained in Section \ref{Sec:Preliminaries}, at a generic round $2\tau$, after having pulled both arms $\tau$ times, any learning policy can either commit to an arm, thereby obtaining $\tout =  T - \tau$, or keep exploring, thereby having as best reachable state one where $\tout =  T - (\tau + 1)$. For this reason, we can essentially view the problem at each state $(\tau,\tau)\in[T]\times[T]$ as a bandit problem with three possible arms, whose expected losses are specified by the tuple 
\begin{equation}\label{Eq:Triple}
    \bigl(\mu_1(T-\tau),\, \mu_2(T-\tau),\, \min\{\mu_1(T-\tau-1), \mu_2(T-\tau-1)\}\bigl)~.
\end{equation}
The first two components refer to the expected loss that any policy would obtain by committing to one of the arms. The third component $\min\{\mu_1(T-\tau-1), \mu_2(T-\tau-1)\}$ is the smallest expected loss that any learning policy would be able to obtain by keeping exploring at the current 
state $(\tau,\tau)$.\\
\newline
Following the proof in \cite{bubeck2013bounded}, we first determine the lower bound due to the gap factor $\Delta$.
To this effect, let us consider the following pair of instances $\muu^{\tau} = (\mu_1(T-\tau),\, \mu_1(T-\tau)+\Delta)$ and $\muu'^{\tau} = (\mu_1(T-\tau)+\Delta,\, \mu_1(T-\tau))$. At this stage, when considering the policy at state $(\tau, \tau)$ we also assume it to have access to $\tau$ independent and identically distributed samples for each arm with expected value parameterized by $\tout=T-\tau$. Clearly, this stationary setting is simpler than the original non-stationary bandit problem, therefore any lower bound for the former carries over to the latter. We are interested in 
the stopping time associated with the exploratory stage, as given below:
\begin{align*}
	\max\Big(\E\big[\tau(2,2\tau)\big],\E'\big[\tau(1,2\tau)\big]\Big) 
	&\geq \frac{1}{2} \E\big[\tau(2,2\tau)\big] + \frac{1}{2} \E'\big[\tau(1,2\tau)\big]\\
	&=\frac{1}{2} \Bigg[ \sum_{s=1}^{\tau} \Pro_{\muu^\tau,\F_{2s}} (I_{2s} = 2) + \Pro_{\muu'^\tau,\F_{2s}} (I_{2s} = 1) \Bigg]~,
\end{align*}
where $\E[\cdot],\E'[\cdot]$ denote the expected value when losses are generated according to $\muu^\tau$ and $\muu'^\tau$, respectively.
Additionally, because of Fact \ref{Fact:PermutationInvariance}, we can consider $\F_{2s} = \{1, X_{1,1}, 2, X_{2,2,}, \dots, 1 , X_{1,2s-1}, 2 , X_{2,2s}\}$ and  $\F'_{2s} = \{2, X_{2,1}, 1, X_{1,2,}, \dots,  X_{1,2s}\}$, which entails that being in state $(\tau,\tau)$ the maximal number of pulls associated with the sub-optimal arm is equal to $\tau$. Finally, still relying on Fact \ref{Fact:PermutationInvariance}, we can assume the sub-optimal arm to be selected only at even rounds in both instances $\muu$ and $\muu'$.\\
\newline
We now leverage Lemma \ref{Le:Tsyabkov} assuming $\iT=1$ in the first environment and $\iT=2$ in the second. This yields
\begin{align}
	&\max\Big(\E\big[\tau(2,2\tau)\big],\E'\big[\tau(1,2\tau)\big]\Big) \geq \frac{1}{2} \sum_{s=1}^{\tau} \exp\Big(-KL\big(\Pro_{\muu^\tau,\F_{2s}}^{\otimes s}, \Pro_{\muu'^\tau,\F'_{2s}}^{\otimes s}\big)\Big) \nonumber\\
	&\;\geq\frac{1}{2} \sum_{s=1}^{\tau} \exp\Big(- s \max_{\sbar \in[1,s]} KL\big(\Pro_{\muu^\tau,\F_{2\sbar}}, \Pro_{\muu'^\tau,\F'_{2\sbar}}\big)\Big) \nonumber\\
	&\;= \frac{1}{2} \sum_{s=1}^{\tau} \exp\big(- 4 s C \Delta^2 \big) \geq \frac{1}{8 C \Delta^2}~. \nonumber
\end{align}
Notice that this result holds for all states $(\tau,\tau)$ since, for the chosen instances, the KL-divergence is invariant with respect to the state: $KL(\Pro_{\mu^s}, \Pro_{\mu'^s}) = C \Delta^2 \; \forall s \in [T]$. So far we have considered only the $\Delta$ factor which allows any policy to commit to the arm having the lower expected loss at $\tout$. We move now to analyze the second reason of commitment. In agreement with Equation \ref{Eq:Triple}, in order to minimize the regret, any policy at state $(\tau,\tau)$ would terminate the round-robin exploration also based on $\mu_1(T-\tau-1) = \min(\mu_1(T-\tau-1), \mu_2(T-\tau-1))$. We can then repeat the same analysis considering the pair of arms with expected value $(\mu_1(T-\tau), \mu_1(T-\tau-1))$,
and replacing $\Delta$ with $\DeltaTilde_\tau = \mu_1(T-\tau-1) - \mu_1(T-\tau) = \ab \Big( \frac{1}{(T-\tau-1)^\rho} - \frac{1}{(T-\tau)^\rho}\Big)$. Coherently with the previous assumption when evaluating a policy at state $(\tau, \tau)$ we can consider having $\tau$ stationary samples for each arm having expected value respectively equal to $\mu_1(T-\tau)$ and $\mu_1(T-\tau-1)$.\\
\newline
This first part of the proof refers to the simpler problem of understanding which index corresponds to the optimal arm given a pair of different arms. Thanks to this result we could already show the $\frac{1}{\Delta^2}$ (respectively $1/\DeltaTilde_\tau^2$) dependency. The lower bound is obtained by considering the first state $(\tau,\tau)\in [T/2]\times [T/2]$ at which any policy would be able to distinguish based on $\Delta$ or $\DeltaTilde_\tau$. This smallest $\tau$ can be defined as
\[
    \tau_0 = \arg\min \left\{\tau \in [T/2]: \tau \geq \min \left( \frac{1}{8 C \Delta^2}, \frac{1}{8 C \DeltaTilde_\tau^2} \right)\right\}.
\]
After $\tau_0$, any arm-elimination policy would stop exploring and would commit to arm $\iout$.
We can now show that the sole knowledge of $\Delta$ (respectively $\DeltaTilde_\tau$) already leads to a lower bound on $\max\Big\{\E\big[\tau(2,2\tau)\big],\E'\big[\tau(1,2\tau)\big]\Big\}$ of the order of $\log(T\Delta^2)/\Delta^2$ (respectively, $\log(T\DeltaTilde_\tau^2)/\DeltaTilde_\tau^2$).
Let us consider the following pair of instances $\muu^{\tau} = (\mu_1(T-\tau),\, \mu_1(T-\tau)+\Delta)$ and $\muu'^{\tau} = (\mu_1(T-\tau),\, \mu_1(T-\tau)-\Delta)$, and notice that
\begin{equation*}
    \max\Big\{\E\big[\tau(2,2\tau)\big],\E'\big[\tau(1,2\tau)\big]\Big\} \geq \E\big[\tau(2,2\tau)\big]~.
\end{equation*}
Following the same steps taken above, being in state $(\tau, \tau)$ implies
\begin{equation}\label{Eq:MaxKL}
    \max\Big\{\E\big[\tau(2, 2\tau)\big],\E'\big[\tau(1, 2\tau)\big]\Big\} \geq \frac{1}{2} \sum_{s=1}^{\tau} \exp\Big(-KL\big(\Pro_{\muu^\tau,\F_{2s}}^{\otimes s}, \Pro_{\muu'^\tau,\F'_{2s}}^{\otimes s}\big)\Big)~.
\end{equation}
Finally, in the considered environments, $KL(\cdot,\cdot)$ is only function of the second arm, specifically %
\begin{equation*}
    KL\big(\Pro_{\muu^\tau,\F_{2s}}^{\otimes \tau}, \Pro_{\muu'^\tau,\F_{2s}'}^{\otimes \tau}\big) = 4 C \Delta^2 \E[\tau(2,\tau)]~.
\end{equation*}
Following the same reasoning as before, the same holds when considering $\DeltaTilde_\tau$. We can then combine the above results together as follows:
\begin{align*}
	\max&\Big\{\E\big[\tau(2, 2\tau)\big],\E'\big[\tau(1, 2\tau)\big]\Big\}\\
	&\geq \frac{1}{2} \Bigg[\E[\tau(2,2\tau)] + \frac{1}{2} \sum_{s=1}^{\tau} \exp\Big(-KL\big(\Pro_{\muu^\tau,\F_{2s}}^{\otimes s}, \Pro_{\muu'^\tau,\F_{2s}'}^{\otimes s}\big)\Big) \Bigg]\\
	&\geq \frac{1}{2} \min_{x\in[0,\tau]} \bigg[ x + \frac{\tau}{2} \exp\Big(-4 C \Delta^2 x\Big) \bigg]\\
	&\geq \frac{1}{8 C \Delta^2} \log\Big(C \Delta^2 \tau / 4 \Big)~.
\end{align*}
The above also holds for $\DeltaTilde_\tau$, so that the resulting lower bound on the number of sub-optimal pulls becomes
\[
    \tsub = \min\left\{\tau \in [T/2]\,: \tau \geq \min \left\{ \frac{1}{8 C \DeltaTilde_\tau^2} \log\Big(C \DeltaTilde_\tau^2 \tau / 4 \Big) , \frac{1}{8 C \Delta^2} \log\Big(C \Delta^2 \tau / 4 \Big) \right\} \right\}~.
\]

%
%
\subsection{Proof of Corollary \ref{Co:RLB}}\label{SuppSec:CoRLB}
The proof is a combination of the result of Theorem \ref{Th:RLB} with the following inequalities
\begin{align*}
    &\sqrt{x} - \sqrt{x-1}>\frac{1}{2\sqrt{x}}
    &\sqrt{x} - \sqrt{x-1}<\frac{1}{2\sqrt{x-1}}
\end{align*}
which hold for $x\geq 1 $.

When considering $\rho=\frac{1}{2}$, we adopted them to obtain a simpler form for the difference $\frac{1}{(T-\tau-1)^\rho} - \frac{1}{(T-\tau)^\rho}$.
We can write
\begin{align*}
    \frac{1}{\sqrt{T-\tau-1}} - \frac{1}{\sqrt{T-\tau}} &= \frac{\sqrt{T-\tau} - \sqrt{T-\tau-1}}{\sqrt{T-\tau}\sqrt{T-\tau-1}} < \frac{1}{2\sqrt{(T-\tau-1)^3}}\\[5pt]
    \frac{1}{\sqrt{T-\tau-1}} - \frac{1}{\sqrt{T-\tau}}  &= \frac{\sqrt{T-\tau} - \sqrt{T-\tau-1}}{\sqrt{T-\tau}\sqrt{T-\tau-1}}> \frac{1}{2\sqrt{(T-\tau)^3}}.
\end{align*}
That said, starting from the proof of Theorem \ref{Th:RLB} by arithmetic calculations we have 
\begin{align*}
    \tsub = \min \Bigg\{ \tau \in [T/2] : \, \tau \geq \ceil*{ \frac{\log(\tau C \Delta^2)}{8 C \Delta^2} }, \ceil*{\frac{ (T-\tau-1)^3}{2 C \ab^2} \log\left(\frac{\tau C \ab^2}{16 (T-\tau)^3}\right)}  \Bigg\}.
\end{align*}
Finally, the $T/2$ factor guarantees the feasibility in agreement with the fact that $T$ is the maximum number of pulls.
%
%

\section{Proofs for Section \ref{Sec:ParamEst}} 

\subsection{Proof of Theorem \ref{Th:LossEstimate}}\label{SuppSec:LossEstimate}
As mentioned in Section \ref{Sec:ParamEst}, the estimation of parameters $\alpha_i,\beta_i$ relies on Bernstein's inequality, which we recall below.
\begin{theorem}\label{Th:Bernstein} Let $X_i(1), \dots, X_i(\tau)$ be $t$ independent random variables with range $[0,U+1]$ and variance $\Var[X_i(s)]$.  Then the following holds:
    \begin{equation*}
        \left|\frac{1}{\tau} \sum_{s=1}^{\tau} X_i(s) - \frac{1}{\tau} \sum_{s=1}^{\tau} \E[X_i(s)]\right|\leq \frac{2(U+1)\,\log1/\delta}{3\tau} + \frac{1}{\tau} \sqrt{2 \log \frac{1}{\delta} \sum_{s=1}^\tau \Var[X_i(s)] }
    \end{equation*}
    with probability at least $1-\delta$.
\end{theorem}
Since the loss random variables in this paper have support in $[0,U+1]$, we can use the fact that $\forall\;i\in\K\;,\tau_i\in[T]$
\begin{equation*}
    \Var[X_i(\tau_i)] \leq (U+1)\,\E[X_i(\tau_i)]~.
\end{equation*}
%
Starting from the estimators defined in Equation (\ref{Eq:Estimators}) and in agreement with Theorem \ref{Th:Bernstein} we can construct the following confidence intervals
    \begin{align*}
        \beta_i + \frac{\alpha_i}{\tau}\sum_{s=1}^\tau \frac{1}{s^\rho} - CB_{\Xh,\tau}(\delta) &\leq \Xh_{i,\tau} \leq \beta_i + \frac{\alpha_i}{\tau}\sum_{s=1}^\tau \frac{1}{s^\rho} + CB_{\Xh,\tau}(\delta)\\ %
        \beta_i + \frac{\alpha_i}{\tau}\sum_{s=\tau+1}^{2\tau} \frac{1}{s^\rho} - CB_{\Xt,\tau}(\delta) &\leq \Xt_{i,\tau} \leq \beta_i + \frac{\alpha_i}{\tau}\sum_{s=\tau+1}^{2\tau} \frac{1}{s^\rho} + CB_{\Xt,\tau}(\delta)~,
    \end{align*}
where we introduced for brevity the following confidence bounds around $\Xh_{i,\tau}$ and $\Xt_{i,\tau}$:
\begin{align*}
        CB_{\Xh,\tau}(\delta) &= \left[ \left(\frac{1}{\tau}\sqrt{\sum_{s=1}^{\tau}\frac{U}{s^\rho}} + \sqrt{\frac{1}{\tau}}\right) \sqrt{2(U+1)\log\frac{1}{\delta}} + \frac{2(U+1)\log 1/\delta}{3\tau} \right]\\
        CB_{\Xt,\tau}(\delta) &= \left[ \left(\frac{1}{\tau}\sqrt{\sum_{s=\tau+1}^{2\tau}\frac{U}{s^\rho}} + \sqrt{\frac{1}{\tau}}\right) \sqrt{2(U+1)\log\frac{1}{\delta}} + \frac{2(U+1)\log 1/\delta}{3\tau} \right].
\end{align*}
Let now $\DX_{i,\tau} = \Xh_{i,\tau}-\Xt_{i,\tau}$. Since $CB_{\Xh,\tau}(\delta) \geq CB_{\Xt,\tau}(\delta)$ we can write
\begin{align*}
        \frac{\alpha_i}{\tau}\left( \sum_{s=1}^n \frac{1}{s^\rho} - \sum_{s=n+1}^{2\tau} \frac{1}{s^\rho} \right) - 2 CB_{\Xh, \tau} &(\delta) \leq \quad \DX_{i,\tau}\\
        &\leq \frac{\alpha_i}{\tau}\left( \sum_{s=1}^n \frac{1}{s^\rho} - \sum_{s=n+1}^{2\tau} \frac{1}{s^\rho} \right) + 2 CB_{\Xh, \tau} (\delta).
    \end{align*}
Solving for $\alpha_i$ and abbreviating  $\ab_i$'s confidence interval $[\ah_i - CB(\delta)\,,\, \ah_i + CB(\delta)]$ by
$\ab_i \in \ah_i \pm CB(\delta)$, we can write
\begin{align*}
        \alpha_i &\in \frac{\tau}{\sum_{s=1}^\tau \frac{1}{s^\rho} - \sum_{s=\tau+1}^{2\tau} \frac{1}{s^\rho}} \left[ \DX_{i,\tau} \pm 2 CB_{\Xh, \tau} (\delta) \right] \\
        &\in \frac{\tau \DX_{i,\tau} }{\sum_{s=1}^\tau \frac{1}{s^\rho} - \sum_{s=\tau+1}^{2\tau} \frac{1}{s^\rho}}\\ &\quad\pm \frac{(2\tau)^\rho}{(2^\rho-1)} \left[ \left(\frac{1}{\tau}\sqrt{\sum_{s=1}^{\tau}\frac{U}{s^\rho}}  + \sqrt{\frac{1}{\tau}} \right) 2\sqrt{2(U+1)\log\frac{1}{\delta}} + \frac{4(U+1)\log1/\delta}{3\tau}\right] \\
    	&\in \ah_{i,\tau} \pm \frac{5\tau^\rho\left(\sqrt{U}+1\right)^2}{\rho} \left[\sqrt{\frac{1}{\tau}\log\frac{1}{\delta}} + \frac{\log1/\delta}{\tau}\right]~,
\end{align*}
which corresponds to the confidence interval in Equation (\ref{Eq:alpha}). In the second step above we have used $2\sqrt{2}/(2^\rho-1) \geq 2\sqrt{2}/(\rho\log 2) \geq 5/\rho $, while in the first step we exploited the following lower bound argument:
\begin{align*}
    	\frac{1}{\tau}\sum_{s=1}^{\tau}\Big(\frac{1}{s^\rho} - \frac{1}{(s+\tau)^\rho}\Big) 
    	&\geq \frac{1}{\tau} \sum_{s=1}^{\tau}\frac{\left(\frac{s+\tau}{s}\right)^\rho -1}{(2\tau)^\rho}\\
    	&= \frac{1}{\tau} \sum_{s=1}^{\tau} \left(\frac{1 + \tau/s}{2\tau}\right)^\rho - \frac{1}{\tau} \sum_{s=1}^{\tau} \frac{1}{(2\tau)^\rho} \\
    	& = \frac{1}{\tau}\sum_{s=1}^{\tau} \left(\frac{1}{2\tau} + \frac{1}{2s}\right)^\rho - \frac{1}{(2\tau)^\rho}\\ 
    	&\geq \frac{1}{(2\tau)^\rho} (2^\rho - 1)~.
\end{align*}
Plugging the above confidence bounds for $\alpha_i$ back into the first equation of the original system of inequalities, we obtain the following result for $\beta_i$:
\begin{align*}
        \beta_i &\in \Xh_{i,\tau} - \frac{\ah_{i,\tau}}{\tau} \sum_{s=1}^\tau \frac{1}{s^\rho} \\
        &\quad \pm \frac{5\tau^\rho}{\rho} \frac{\tau^{1-\rho}}{(1-\rho)\tau} \left[ \left(\sqrt{\sum_{s=1}^{\tau}\frac{U}{s^\rho}} + \sqrt{\frac{1}{\tau}} \right) \sqrt{2(U+1)\log\frac{1}{\delta}} + \frac{2(U+1)\log1/\delta}{3\tau}\right]  \\
        &\in \Xh_{i,\tau} - \frac{\ah_{i,\tau}}{\tau} \sum_{s=1}^\tau \frac{1}{s^\rho} \pm \frac{5}{\rho(1-\rho)} \bigg[\left(\sqrt{U}+1\right)^2 \sqrt{\frac{1}{\tau}} \sqrt{2\log\frac{1}{\delta}} + \frac{2(U+1)\log1/\delta}{3\tau}\bigg]~,
    \end{align*}
which in turn corresponds to the confidence interval in Equation (\ref{Eq:beta}). The confidence interval associated with the expected loss $\mu_{i}(\tout)$ incurred by arm $i\in\K$ after $\tout$ pulls then follows from properly combining the above bounds, taking into account (\ref{e:mu}). Finally, Theorem \ref{Th:LossEstimate} is obtained by an union bound over $i,j\in\K,\tau,\tout\in[T]$ that allows us to state that for all $i,j$, $\tau$, and $\tout$,  $\mu_i(\tout) > \mu_j(\tout)$ holds with probability at least $1-\delta$,  starting from $\muh_{i,\tau}(\tout) > \muh_{j,\tau}(\tout) + 2 CB_{\mu,\tau}(\delta)$.
%
%
\section{Proof for Section \ref{Sec:LearningProb}}\label{SuppSec:ETC_RB}

\subsection{Proof of Theorem \ref{Th:FC_ETC}}
Let us recall the confidence bound around $\mu$:
\[
\text{CB}_{\mu,\tau}(\delta) = \frac{10\left(\sqrt{U}+1\right)^2}{(1-\rho)\rho}  \Bigg[ \frac{\log(4 \tau T^2)}{\tau} + \sqrt{\frac{1}{\tau}\log(4 \tau T^2)}\Bigg].
\]
According to Algorithm \ref{Alg:FC_ETC}, we have that the exploration phase terminates once the confidence interval at $\tout$ stops overlapping with the one containing the smallest estimated loss, that is when  $\muh_{2,n_0}(\tout) - \muh_{1,n_0}(\tout) \geq 2 CB_{\mu, n_0}(1/T)$. According to Theorem \ref{Th:LossEstimate}, this translates into the following condition on the gap parameter $\Delta$:
\[
    \Delta - 2 CB_{\mu,n_0}(1/T) \geq \muh_{2,n_0}(\tout) - \muh_{1,n_0}(\tout) \geq 2 CB_{\mu,n_0}(1/T)~.
\]
This implies
%
\begin{equation*}
        \Delta \geq \frac{40\left(\sqrt{U}+1\right)^2}{(1-\rho)\rho}  \Bigg[ \frac{\log(4 n_0 T^2)}{n_0} + \sqrt{\frac{1}{n_0}\log(4 n_0 T^2)}\Bigg].
    \end{equation*}
Solving for $n_0$ and lower bounding the RHS by removing term $\log( 4n_0 T^2)/n_0$
%
%
we obtain that
\[
    n_0 > \frac{1600 (\sqrt{U}+1)^4}{\rho^2(1-\rho)^2}\frac{\log(4 n_0 T^2)}{\Delta^2}~.
\]
For this being feasible we must also have $n_0 \leq \frac{T}{K}$. Finally, the optimality condition directly follows from Theorem \ref{Th:RLB}. According to the above results we have that
\[
    n_0 = \min\left\{ \ceil*{\frac{1600 (\sqrt{U}+1)^4}{\rho^2(1-\rho)^2}\frac{\log(4 n_0 T^2)}{\Delta^2}}, \ceil*{\frac{T}{2}} \right\}.
\]
Notice that only when the minimum corresponds to the first argument ETC can guarantee $\mu_{\iout}(T-n_0)=\mu*(T-n_0)$ with probability at least $1-\delta$. In the second case the only available guarantee is that $\mu_{\iout}(T-n_0) \leq \mu*(T-n_0) + 2 CB_{\mu,n_0}(1/T)$. The statement follows by observing that $\mu*(T-n_0) - \mu_{\iT}(T) = \alpha \left(\frac{1}{(T-n_0)^\rho} - \frac{1}{T^\rho}\right)$.
%
%

\subsection{Proof of Theorem \ref{Th:RESTSURE_RB}}\label{SuppSec:RESTSURE_RB}
We only present the analysis for the first arm $\sigma(1)\in\As_1$. The same line of reasoning holds for the other arms, the only difference being that set $\As_s$ becomes $\As_s = \K \setminus \{\sigma(1),\dots,\sigma(s-1)\}$.\\
According to the elimination condition in Line \ref{AlgLine:DAEElimination} of Algorithm \ref{Alg:RESTSURE}, the exploration over arm $\sigma(1)$ terminates as soon as the following condition is satisfied
\begin{align*}
    &\min_{i\in\K, m \in [n_{\sigma(1)}, T-(K-1)n_{\sigma(1)}]} \big( \mu_{\sigma(1)}(m) - \mu_i(m)\big) \\
    &\hspace{5em} \geq \frac{40\left(\sqrt{U}+1\right)^2}{(1-\rho)\rho} \Bigg[ \frac{\log(n_{\sigma(1)} K^2 T^2)}{n_{\sigma(1)}} + \sqrt{\frac{1}{n_{\sigma(1)}}\log(n_{\sigma(1)} K^2  T^2)}\Bigg]~.
\end{align*}
Similar to the proof of Theorem \ref{Th:FC_ETC}, solving the above for $n_{\sigma(1)}$ gives
\begin{align*}
    n_{\sigma(1)} \geq \frac{1600 (\sqrt{U}+1)^4}{\rho^2(1-\rho)^2}\frac{\log(n_{\sigma(1)} k^2 T^2)}{\Big(\min_{i\in\K, m \in [n_{\sigma(1)}, T-(K-1)n_{\sigma(1)}]} \big( \mu_{\sigma(1)}(m) - \mu_i(m) \big)\Big)^2}~.
\end{align*}
The other case where arm $\sigma(1)$ is (implicitly) removed corresponds to the case where \RS prefers to commit to $i\in\K$. This occurs in one of the following cases:
\begin{align*}
    \begin{cases}
        \mu^*\big(T-(K-1)&(n_{\sigma(1)}+1)\big) -\mu^*\big(T-(K-1)(n_{\sigma(1)})\big) \geq \\
        &\quad \frac{20\left(\sqrt{U}+1\right)^2}{(1-\rho)\rho} \Bigg[ \frac{\log(n_{\sigma(1)} K^2 T^2)}{n_{\sigma(1)}} + \sqrt{\frac{1}{n_{\sigma(1)}}\log(n_{\sigma(1)} K^2  T^2)}\Bigg]\\
        \min_{i\in\K\setminus\{\sigma(1)\}} &\big( \mu_{i}(T-n_{\sigma(1)}) - \mu_{\sigma(1)}(T-n_{\sigma(1)})\big) \geq \\ 
        &\quad  \frac{20\left(\sqrt{U}+1\right)^2}{(1-\rho)\rho} \Bigg[ \frac{\log(n_{\sigma(1)} K^2 T^2)}{n_{\sigma(1)}} + \sqrt{\frac{1}{n_{\sigma(1)}}\log(n_{\sigma(1)} K^2  T^2)}\Bigg]~,
    \end{cases}
\end{align*}
these inequalities corresponding to the conditions specified in Line \ref{AlgLine:StopExploration} and Line \ref{AlgLine:DAECommit} of Algorithm \ref{Alg:RESTSURE}, respectively. Solving these for $n_{\sigma(1)}$ yields:
\begin{align*}
    \begin{cases}
        n_{\sigma(1)} &\geq \frac{1600 (\sqrt{U}+1)^4}{\rho^2(1-\rho)^2}\frac{\log(n_{\sigma(1)} K^2 T^2)}{\Big(\mu^*\big(T-(K-1)(n_{\sigma(1)}+1)\big) -\mu^*\big(T-(K-1)(n_{\sigma(1)})\big) \Big)^2}\\
        n_{\sigma(1)} &\geq \frac{1600 (\sqrt{U}+1)^4}{\rho^2(1-\rho)^2}\frac{\log(n_{\sigma(1)} K^2 T^2)}{\Big(\min_{i\in\K\setminus\{\sigma(1)\}} \big( \mu_{i}(T-n_{\sigma(1)}) - \mu_{\sigma(1)}(T-n_{\sigma(1)})\big) \Big)^2}~.
    \end{cases}
\end{align*}
Notice that there is a substantial difference between the two conditions above. In the first case, thanks to the constructed confidence bounds (Theorem \ref{Th:LossEstimate}) which hold with high probability, \RS can guarantee that $\iout\in\arg\min_{i\in\K}\mu_i\big(T-(K-1)n_{\sigma(1)}\big)$ with the same probability. Conversely, in the second case, the only available guarantee is that $\mu_{\iout}(T-(K-1)n_{\sigma(1)}) \leq \min_{i\in\K}\mu_i(T-(K-1)n_{\sigma(1)}) + 2 CB_{\mu,n_{\sigma(1)}}(\delta)$. Finally, in both cases we would have $\tout = T-(K-1)n_{\sigma(1)}$.

It is important to remark here that $\sigma(1)$ is defined as $\sigma(1) = \arg\min_{j\in\As_0} n_{\sigma(j)}$, which is solely a function of the problem parameters, rather than an algorithm-dependent quantity. 

The statement of the Theorem is then obtained by iterating this very same reasoning to all arms $\sigma(2),\dots,\sigma(K-1),\iout$ in turn.
%
%
\subsection{Proof of Corollary \ref{Co:RESTSURE_Bound}}\label{SuppSec:RESTSURE_CO}
The proof directly follows from the one mentioned in the previous section. Following the notation adopted in the statement of Theorem \ref{Th:RESTSURE_RB}, we set $n_{\sigma(1)}=n_2$ for the number of pulls necessary to eliminate arm $2\in\K$ based on the condition displayed in Line \ref{AlgLine:DAEElimination} of Algorithm \ref{Alg:RESTSURE}. Hence, for arm $2\in\K$, we can write
\begin{equation*}
    \Delta \geq \frac{40\left(\sqrt{U}+1\right)^2}{(1-\rho)\rho} \Bigg[ \frac{\log(4 n_2 T^2)}{n_2} + \sqrt{\frac{1}{n_2}\log(4 n_2 T^2)}\Bigg].
\end{equation*}
Similar to what we did in the proof of Theorem \ref{Th:FC_ETC}, solving the above for $n_2$ yields
\begin{equation*}
    n_2 \geq 25600 (\sqrt{U}+1)^4\frac{\log(4 n_2 T^2)}{\Delta^2}~.
\end{equation*}
We can now analyze the implicit elimination condition that corresponds to \RS committing to arm $\iout$. When this is the case, we have
\begin{equation*}
    n_{\iout} \geq
    25600 (\sqrt{U}+1)^4 \frac{\log(4 n_{\iout} T^2)}{\ab^2\Bigg(\frac{1}{\big(T-(n_{\iout}+1)\big)^{\frac{1}{2}}} - \frac{1}{\big(T-n_{\iout}\big)^{\frac{1}{2}}} \Bigg)^2}
\end{equation*}
which, similar to the derivation contained in the proof of Corollary \ref{Co:RLB}, gives
\[
    n_{\iout} \geq 25600 (\sqrt{U}+1)^4 \frac{(T-n_{\iout})^3}{\ab^2} \log(4 n_{\iout} T^2)~.
\]
Combining the above results we obtain
\[
    n_0 = \min \left\{ \frac{c_\rho}{\Delta^2} , c_\rho \frac{(T-n_0)^3}{\alpha^2}\right\},
\]
where $c_\rho = 25600 (\sqrt{U}+1)^4 \log(4 n_0 T^2)$. As already discussed, if the commitment occurs due to the second condition (Line \ref{AlgLine:StopExploration} of Algorithm \ref{Alg:RESTSURE}), the tighter 
regret bound we can obtain satisfies
\begin{equation*}
    \mu_{\iout}(T - n_0) < \mu^*(T-n_0) + 2 CB_{\mu,n_{\iout}}(\delta).
\end{equation*}
Conversely, when \RS commits to $\iout$ based on the first condition (Line \ref{AlgLine:Commit} of Algorithm \ref{Alg:RESTSURE}), we have the tighter result $\beta_{\iout} = \beta_1$.